\pdfoutput=1
\documentclass[11pt,a4paper]{article}
\RequirePackage[round]{natbib}

\usepackage{booktabs}

\usepackage{times}
\usepackage{xcolor}
\usepackage{soul}
\usepackage[utf8]{inputenc}
\usepackage[small]{caption}

\usepackage{color}
\usepackage{amsfonts,amsthm,amsmath}
\usepackage{tikz}
\usetikzlibrary{arrows,shapes, calc, fit, positioning}
\usepackage{tkz-graph}
\usepackage[textsize=tiny]{todonotes}

\newcommand{\bbR}{\mathbb{R}}

\newcommand{\calP}{\mathcal{P}}

\newcommand{\I}{\mathcal{I}}
\newcommand{\calU}{\mathcal{U}}

\newcommand{\Pol}{\ifmmode\text{\textsf{P}}\else\textsf{P}\fi}
\newcommand{\NP}{\ifmmode\text{\textsf{NP}}\else\textsf{NP}\fi}
\newcommand{\NPc}{\ifmmode\text{\textsf{NP}-complete}\else\textsf{NP}-complete\fi}
\newcommand{\NPC}{\ifmmode\text{\textsf{NPC}}\else\textsf{NPC}\fi}

\theoremstyle{plain}
	  \newtheorem{theorem}{Theorem}[section]
	  \newtheorem{cor}[theorem]{Corollary}
	  \newtheorem{lemma}[theorem]{Lemma}
	  \newtheorem{prop}[theorem]{Proposition}

\theoremstyle{definition}

	  \newtheorem{example}[theorem]{Example}
	  
\theoremstyle{remark}
	  
\title{Chore division on a graph}
\author{Sylvain Bouveret$^a$, Katar\'ina Cechl\'arov\'a$^b$, Julien Lesca$^c$\\ \\
$^a$Univ. Grenoble Alpes, CNRS, Grenoble INP, LIG, Grenoble, France\\
$^b$P.J. \v{S}af\'arik University, Ko\v sice, Slovakia\\ 
$^c$Universit\'e Paris-Dauphine, PSL, CNRS, LAMSADE, Paris, France}
\date{November 28, 2018}
\begin{document}

\maketitle

\begin{abstract}
The paper considers fair allocation of indivisible nondisposable  items that generate disutility (chores). We assume that these items are  placed in the vertices of a graph and each agent's  share has to form a connected subgraph of this graph. Although a similar model has been investigated before for  goods, we show that the goods and chores settings are inherently different. In particular, it is impossible to derive the solution of the chores instance from the solution of its naturally associated fair division instance.
We consider three common fair division solution concepts, namely proportionality, envy-freeness and equitability, and two individual disutility aggregation functions: additive and maximum based. 
We show that deciding the existence of a fair allocation is hard even if the underlying graph is a path or a star. We also present some efficiently solvable special cases for these graph topologies. % and explore approximate fairness in this setting.

\end{abstract}

\section{Introduction}
Fair division of goods and resources is a practical problem in  many situations and a popular research topic in Economics, Mathematics and Computer science. 
Sometimes, however, the objects that people have to deal with are undesirable, or, instead of utility create some cost. 
Imagine that a cleaning service firm allocates to its teams  a set of offices, corridors, etc in a building. Each team has some idea of how much effort each room requires. The cost of the whole assignment for a team may be in the form of time the team will have to spend on the job, and this might depend on whether the work in all assigned  rooms can start simultaneously or whether they have to be treated one after another.  Moreover, for practical reasons, it is desirable that each team's assignment is a contiguous set of rooms. As another example consider a firm that supervises the operation of the computer network during a conference. Each of its employees has to choose one of the possible shifts, but shifts scheduled at different times  of the conference may incur different  opportunity cost for various persons. Moreover, we can assume that everybody prefers to have just one uninterrupted period to spend at work.

The  constraints described  above could be represented  by an undirected graph whose vertices are rooms or time intervals and there is an edge between two vertices if the respective rooms are adjacent; or the corresponding  time periods  immediately follow each other. Each agent should obtain a connected piece of the underlying neighborhood  graph. Complete graphs correspond to the 'classical' case with no connectivity constraints. By contrast, paths and stars represent the simplest combinatorial structures, and yet such graphs can be used to model a rich variety of situations. For example, a path may represents successive time intervals and a star corresponds to a building with a central foyer and mutually non-connected rooms accessible from this foyer.

\paragraph{Related work}

The mathematical theory of fair division started with the seminal
paper of \cite{Steinhaus}.
% Steinhaus [\cite{Steinhaus}].
Although originally most researchers focused on \emph{divisible} items
(the topic is also known as the \emph{cake-cutting problem}), fair division of
indivisible goods has also received a considerable amount of attention in
the traditional and computational social choice literature. The
interested reader can for instance read the survey by
\cite{bouveret-chap} for an overview of this topic.

Several recent papers combine graphs and fair division of indivisible items by assuming that agents are located in the vertices of a graph. For instance, \cite{EndrissAAAI07} and \cite{GourvesLesca17} consider the
case where each agent has an initial endowment of goods and can trade with her neighbors in the graph. The authors study outcomes that can be achieved by a sequence of mutually beneficial deals.

In the paper by \cite{BouveretCechl2017} for the first time a
graph-based constraint on agents' bundles has been imposed. The
authors model the items as vertices of a graph and they require that
each agent should receive a bundle that forms a connected
subgraph. They show that even if the underlying graph is a path the problems to decide whether there exist proportional or envy-free divisions with connected shares are \NP-complete. In case of stars,
envy-freeness is also intractable, but a proportional division can be
found by a polynomial algorithm. In addition, \cite{BouveretCechl2017}
study maximin (MMS) allocations, a fairness notion
introduced by \cite{Budish11}. They show that a MMS  allocation
always exists if the underlying graph is a tree, provide a polynomial
algorithm to find one in this case and show that a cycle may admit
no MMS allocation.

This work inspired other authors who also considered allocations on a
graph with the additional constraint that each bundle has to be
connected. \cite{Suksompong17} considers only paths and studies
approximately fair (proportional, envy-free and equitable) allocations
up to an additive approximation factor. He shows that for all the
three fairness notions there is a simple approximation guarantee
derived from the maximum value of an item and that for proportionality, as well as equitability, an allocation achieving this bound can be
computed efficiently.

\cite{Bilo_et_al} deal with two relaxations of envy-freeness: envy-free
up to one good, briefly EF1 (an agent does not think that another
agent's bundle, possibly with one of its outer items removed, is more
valuable than her own bundle) and envy-free up to two outer goods,
briefly EF2.  They characterize graphs admitting allocations
fulfilling these notions and provide efficient algorithms to find such
allocations.

\cite{LoncTruszczynskiIJCAI2018} study MMS allocations of goods. For the underlying graph being a cycle they
identify several cases  when MMS-allocations always exist
(\textit{e.g.} at most three agents and at most 8 goods, the number of goods
not exceeding twice the number of agents, fixed number of agent types)
and provide results on allocations guaranteeing each agent a certain
portion of her maximin share.

\cite{IgarashiPeters2018} study Pareto-optimality. They show that for
paths and stars a Pareto optimal allocation can be found efficiently,
but the problem is \NP-hard even for trees of bounded pathwidth. They
also show that it is NP-hard to find a Pareto-optimal allocation that
satisfies maximin share even on a path.

It is worth noting that although the study of connected fair division
is relatively recent in the context of indivisible items, there is an
important literature on the contiguity requirement in the context of
cake-cutting. From the great number of various results we consider
among the most interesting ones the contrast between the proven
existence of envy-free \citep{Stromquist80} and equitable
\citep{AumannDomb2010,CechlarovaDobosPillarova} divisions with
connected pieces and, on the other hand, the nonexistence of finite
algorithms for computing them
\citep{Stromquist80, CechlarovaPillarova2012b}.

\smallskip

Beyond considering the connectivity constraint, an important aspect in
which our work departs from the mainstream literature on fair division
is the fact that we consider negative items (chores).

Chore division of divisible goods was mentioned for the first
time by \cite{Gardner}. %Gardner [\citeyear{Gardner}].
Although straightforward modifications of some algorithms for positive
utilities can also be applied to the chore division context (\textit{e.g.}, the
Moving Knife algorithm for proportional divisions), it happens more
often that chore division problems are more involved than their
fair division counterparts. For example, the discrete algorithm for
obtaining an envy-free division of divisible chores for three persons
by Oskui \cite[pages 73-75]{Robertson98} needs nine cuts and the
procedure based on using four moving knifes makes 8 cuts
\citep{PetersonSu2002}, while in the Selfridge's algorithm for
envy-free (positive) division \citep{Woodall1980} five cuts
suffice. If the number of agents is 4, the moving-knife procedure by
\cite{BTZ} % \textit{et al.} [\cite{BTZ}]
needs 11 cuts, while the first algorithm for envy-free division of
chores, given by \cite{PetersonSu2002}, needs 16
cuts. % Peterson and Su [\cite{PetersonSu2002}].

The fact that chore division has been given much less attention in
research is mirrored also in monographs on fair division. For example,
\cite{Robertson98} only deal with chores in Section 5.5; of Chapter 7
on cake cutting in the book \emph{Economics and Computation} edited by
\cite{EconomicsComputation} only Section 7.4.6 treats chores and
Chapter 12 on fair division of indivisible goods in the \emph{Handbook of
Computational Social} Choice by \cite{bouveret-chap} %HandbookComSoc}
does not mention chores at all.

Of the more recent works on chore division let us mention
\cite{Caragiannis2012} who deal with divisible and indivisible goods
and chores from the point of view of the price of fairness for three
fairness notions. \cite{HeydrichStee15} consider the price of fairness
for the fair division of infinitely divisible
chores. \cite{AzizRauchSchryenWalsh} also deal with chores; the
considered fairness notion is maximin share guarantee.
In the divisible chores setting, \cite{Dehghani2018} give the first discrete and bounded protocol for envy-free chore division problem and \cite{FarhadiHajighayiIJCAI2018} prove the $\Omega(n\log\ n)$ lower bound for the number of queries  in a proportional protocol for chores.

\smallskip

Finally, in the context of indivisible items, where the existence of fair allocations
cannot be ensured, it is rather natural to study  approximations of the fairness
criteria.  \cite{Markakis2011WorstCase} prove a
worst case guarantee on the value that every agent can have and they
propose a polynomial algorithm for computing allocations that achieve
this guarantee. By contrast, they show that if $\Pol \ne \NP$ there is
no polynomial algorithm to decide whether there exists an allocation
where each agent can get a bundle worth at least $1/cn$ for any
constant $c\ge 1$. \cite{Lipton04} focus on the concept of
envy-freeness. They show that there exists an allocation with maximum
envy not exceeding the maximum marginal utility of a good. However,
the problem of computing allocations with minimum possible envy is
hard, even in the case of additive utilities. 

\paragraph{Our contribution}

In this paper, we extend the work of \cite{BouveretCechl2017} about
fair division of goods on a graph. We also use the connectivity
constraints defined by a graph on the items. However, we deal with
\emph{nondisposable undesirable} items, often called \emph{chores}.
We use three classical fairness criteria, namely proportionality,
envy-freeness and equitability, and two different individual
disutility aggregation functions: additive and maximum based. We show
that dealing with goods and chores is inherently different, in
particular, it is impossible to transform a chore instance simply by
negating the utility values and applying the algorithm that works for
goods. 

Then we investigate the complexity of the problems to
find a fair allocation of chores. It is known that these problems are hard on complete graphs in the additive case, but the maximum-based case, as far as we know, has not been studied before. Therefore, we complement the picture by providing efficient algorithms for proportionality and equitability, and show that envy-freeness leads to an \NP-complete problem. Further, we  concentrate on two special classes of graphs: paths and stars.

In more detail, we provide a general reduction for paths that directly
implies \NP-completeness of the existence problems for all the considered fairness criteria and both disutility aggregations. Moreover, by a very small modification of the reduction
we obtain that these problems are hard even in the binary
case \textit{i.e.}, when disutility values for chores are either 0 or 1. 
%We then turn to additive approximation and provide polynomial-time algorithms as well as intractability results.

By contrast, if the underlying graph is a star, we propose an
efficient algorithm, based on bipartite matching techniques, to decide
whether a valid allocation exists such that each agent has disutility
0. This in turn implies that envy-freeness and equitability criterion
admit efficient algorithms for decision problems in the binary
case. Matching techniques lead to efficient algorithms also in the
maximum-based case, even when disutilities are not restricted to be binary. In the additive case we provide an efficient algorithm for proportionality. On the other hand, it is \NP-complete to decide the
existence of envy-free or equitable valid allocations on a star.

\medskip

This paper is organized as follows. In Section~\ref{sec:model} we
introduce the model of connected fair division of indivisible chores
and the definitions of the various fairness %and approximate fairness
criteria we use in the paper. Our technical results are presented in
Sections~\ref{sec:complete}, \ref{sec:paths}  and \ref{sec:stars} which respectively deal
with the cases where the underlying graph is a complete graph, a path and a
star. Table~\ref{tab:overview} shows an overview of the results introduced in this paper. Finally, we discuss the results and some open problems in
Section~\ref{sec:conclusion}.

\begin{table}[ht]
\begin{center}
\begin{tabular}{lcccccc}
      \toprule
      & \multicolumn{2}{c}{complete graph} & \multicolumn{2}{c}{path} &
\multicolumn{2}{c}{star}\\
  & additive & maximum & additive & maximum  & additive & maximum \\
      \midrule
proportionality & \NPC & \Pol & \NPC & \NPC$^a$ & \Pol & \Pol\\
envy-freeness & \NPC & \NPC & \NPC & \NPC$^a$ & \NPC$^b$ & \NPC$^b$\\
equitability     & \NPC & \Pol & \NPC & \NPC$^a$  & \NPC & \Pol\\
      \bottomrule
    \multicolumn{7}{l}{\footnotesize $^a$ Even with binary disutilities}\\
    \multicolumn{7}{l}{\footnotesize $^b$ Polynomial with strict disutilities}
    \end{tabular}
\end{center}

\vskip1ex
\caption{Overview of the complexities of the existence problems}\label{tab:overview}
\end{table}

\section{Model}\label{sec:model}

$N=\{1,2,\dots,n\}$ 
is the set of agents, $G=(V,E)$ is an undirected graph.
%, where $m=|V|$. 
Vertices $V$ represent objects, and they are interpreted as nondisposable chores. Each agent $i\in N$ has 
a non-negative disutility (cost, regret) function $u_i: V \rightarrow \bbR_{+}$. The $n$-tuple of   disutility functions is denoted by 
$\calU$. Let $m$ denote the number of vertices of $G$ (chores).  
%For a positive integer $k$ we denote the set $\{1,2,\dots,k\}$ by $[k]$.

An instance of {\sc Connected Chore Division} {\sc CCD} is a triple $\I=(N,G,\calU)$. When we
shall occasionally talk about  problems with positively interpreted utility, we shall call them {\sc Connected Fair Division} problems, briefly {\sc CFD}.

Any subset  $X \subseteq V$ is called a {\em bundle}. We consider two disutility aggregation functions. In the {\em additive} case the disutility agent $i$ derives from bundle $X$ is equal to the sum of the disutilities of the objects that form the bundle, \textit{i.e.}  $u_i^{add}(X)=\sum_{v \in X}u_i(v)$. In the {\em maximum-based} case the disutility of a bundle is derived from the maximum disutility of an object in the bundle, \textit{i.e.} $u_i^{max}(X)=\max\{u_i(v) \mid v\in X\}$. If the aggregation function  is not specified or if it is clear from the context, the superscript may be omitted.  %To denote the  disutility of a agent for the whole graph we use $U_i^{add}$ and $U_i^{max}$, respectively. 
In the maximum-based extension we shall also consider an important {\em binary} case when the disutilities of agents for objects are either 1 or 0. The binary case represents the situation of  agents finding some objects negative without expressing the ``degree of negativity'' and some other objects  bring them neither  nuisance nor  joy. 

We assume  that the disutilities are normalized. This means that  $u_i^{add}(V)=1$ and $u_i^{max}(V)=1$ for each agent $i$ in the additive and  in the maximum-based case, respectively. 

\iffalse
We say that two agents $i,j \in N$ are of the {\em same type} if $u_i(v)=u_j(v)$ for each $v \in V$.  The number of agent types in a given instance of {\sc CCD} is denoted by $p$.
\fi

An \emph{allocation} is a function $\pi:N \rightarrow 2^V$ assigning each agent a bundle of objects. An allocation  $\pi$ is {\em valid} if:
\begin{itemize}\itemsep0pt
\item for each agent $i \in N$, bundle $\pi(i)$ is connected  in $G$;
\item  $\pi$ is complete, \textit{i.e.}, $\bigcup_{i\in N}\pi(i)=V$  and;
\item no item is allocated twice, so that $\pi(i) \cap \pi(j) =\emptyset$ for each pair of distinct agents $i,j \in N$.
\end{itemize}
We say that a valid chore allocation $\pi$ is:
\begin{itemize}
%proportionality
\item
{\em proportional} if $u_i(\pi(i)) \le \frac{1}{n}$ for all $i \in N$;
%envy-freeness
\item
{\em   envy-free} if $u_i(\pi(i)) \le u_i(\pi(j))$ for all $i,j \in N$; 
\item
{\em equitable} if $u_i(\pi(i)) = u_j(\pi(j))$ for all $i,j \in N$. 
\end{itemize}

Let us remind the reader that the corresponding notions of proportionality and envy-freeness for {\sc CFD} are defined by reversing the respective inequalities. Equitability is defined in the same way in both cases, but since no results for goods with connected bundles have been published yet, we shall try to close this gap.

Notice that with maximum-based disutility aggregation, proportionality does not have a similar interpretation as in the additive case, where dividing the disutility by the number of agents  corresponds to sharing the total burden. Still, we shall use this term also in the case when agents care for the worst item in their bundle, meaning that we seek an allocation that restricts the disutility by the same threshold for everybody. 

Taking into account the large number of existing results concerning the fair division of goods, one could be tempted to try and adapt these results to the case of chores. First, observe that if in a fair division instance  there are more agents than items, no proportional and envy-free allocation can exist, which is not necessarily the case with chores. Further, although it seems natural to transform a chore division instance  to a `dual' fair division instance simply by replacing   each disutility $u_i(v)$ by a `reverse' desirable utility $M-u_i(v)$ for each agent $i$ and each object $v$ ($M$ is a sufficiently large real number), we show that the properties of the related  instances do not translate. 

\begin{example}
Let us consider the CCD instance ${\cal I}$ with three agents 1,2,3 and four vertices $v_1,v_2,v_3,v_4$ arranged on a path in this order and disutilities given in the left half of Table \ref{t_1}. Its right half shows the  utilities for the `dual` {\sc CFD}  instance ${\cal I}'$ obtained for $M=10$.\footnote{For typographical reasons, the disutilities of agents in this and the following example are not normalized to 1, but to 10 in the {\sc CCD} instances (tables in the left) and to 30 in the corresponding dual {\sc CFG} instances (tables in the right).  }

\begin{table}[ht]
  \begin{center}
    \begin{tabular}{ccccc}
      \toprule
      & $v_1$ & $v_2$ & $v_3$ & $v_4$\\
      \midrule
      agent 1 & 6 & 4 & 0 &0\\
      agent 2 & 7 & 0 & 1 &2\\
      agent 3 & 5 & 0 & 0 &5\\
      \bottomrule
    \end{tabular}
    \qquad
    \begin{tabular}{ccccc}
      \toprule
      & $v_1$ & $v_2$ & $v_3$ & $v_4$\\
      \midrule
      agent 1 &  4 &  6 & 10 & $10^\star$\\
      agent 2 & $3^\star$ & $10^\star$ & 9 &8\\
      agent 3 & 5 & 10 & $10^\star$ &5\\
      \bottomrule
    \end{tabular}
  \end{center}
  \caption{The  {\sc CCD} (left) and {\sc CFD} (right) instances for proportionality notion}\label{t_1}
\end{table}   

\noindent A proportional chore allocation should give each agent a bundle of disutility at most 10/3. ${\cal I}$ does not admit a valid proportional allocation, as nobody is willing to take vertex $v_1$. In the dual {\sc CFD} instance ${\cal I}'$  a proportional valid allocation exists, simply give agent 1 vertex $v_4$, agent  2 bundle $\{v_1,v_2\}$ and agent 3 vertex $v_3$. This allocation is depicted with stars in the right half of Table \ref{t_1}.
\end{example}

\begin{example}
Now slightly change the disutilities of agent 3; the new {\sc CCD}  instance and its dual {\sc CFD} instance are  given in Table~\ref{t_2}.

\begin{table}[htbp]
  \begin{center}
    \begin{tabular}{ccccc}
      \toprule
      & $v_1$ & $v_2$ & $v_3$ & $v_4$\\
      \midrule
      agent 1 & 6 & 4 & $0^\star$ & $0^\star$\\
     agent 2 & 7 & $0^\star$ & 1 &2\\
      agent 3 & $0^\star$ & 5 & 5 &0\\
      \bottomrule
    \end{tabular}
    \qquad
    \begin{tabular}{ccccc}
      \toprule
      & $v_1$ & $v_2$ & $v_3$ & $v_4$\\
      \midrule
      agent 1 & 4 & 6 & 10 &10\\
      agent 2 & 3 & 10 & 9 &8\\
      agent 3 & 10 & 5 & 5 &10\\
      \bottomrule
\end{tabular}
    \caption{The  {\sc CCD} (left) and {\sc CFD} (right) instances for envy-freeness.}\label{t_2}
  \end{center}
\end{table}   

\noindent Now ${\cal I}$ has an envy-free valid allocation, namely $\pi(1)=\{v_3,v_4\}$,  $\pi(2)=\{v_2\}$ and $\pi(3)=\{v_1\}$, again depicted with stars in the left half of Table \ref{t_2}.
 However, there is no envy-free valid allocation in the dual fair division instance ${\cal I}'$.
To see this, let us first realize that as there are three agents and four items, exactly one of the agents has to receive a bundle consisting of two vertices. Since each allocated bundle has to be connected, there are exactly three such two-elements bundles: $\{v_1,v_2\}$, $\{v_2,v_3\}$ and $\{v_3,v_4\}$.
One can see that each such bundle has utility strictly greater  than 10 for at least two agents, and as each agent values individual vertices at not more than 10, there will always be somebody envying the agent receiving the two-element bundle.
\end{example}

We see that similarly as in the 'classical' case of indivisible goods without connectivity constraints, the existence  of allocations fulfilling the above definitions is not ensured in general. Therefore we shall deal also with approximate fairness. We can ask whether  for a given $c\ge 1$, a valid chore allocation $\pi$ exists such that each agent $i \in N$ receives a bundle such that $u_i(\pi(i)) \le c\cdot\frac{1}{n}$,  or such that $u_i(\pi(i))\le c\cdot u_i(\pi(j))$ or $u_j(\pi(j))/c\le u_i(\pi(i))\le c\cdot u_j(\pi(j)$ holds for each pair of agents $i,j$.
Later in this paper we shall see that even for paths the problems to decide whether a valid allocation exists such that the disutility of each agent equals 0 is \NP-hard. This immediately implies intractability of the just formulated problems for any $c>1$.

\iffalse
Therefore we shall consider approximate fairness up to an additive factor and for special graphs we shall provide tight bounds and  efficient algorithms to compute  allocations   achieving this bound.

We say that  a valid chore allocation $\pi$ is
\begin{itemize}
%proportionality
\item 
{\em  $\varepsilon$-proportional} if $u_i(\pi(i)) \le \frac{1}{n}+\varepsilon$ for all $i \in N$;
%envy-freeness
\item
{\em  $\varepsilon$-envy-free} if $u_i(\pi(i)) \le u_i(\pi(j))+\varepsilon$ for all $i,j \in N$;
\item
{\em $\varepsilon$-equitable} if $|u_i(\pi(i))- u_j(\pi(j))|\le \varepsilon$ for all $i,j \in N$. 
\end{itemize}
\fi

We will consider the following computational problems that all 
take an instance $\I=(G,N,\calU)$ of CCD as their input. {\sc Prop-CCD}, {\sc EF-CCD} and {\sc EQ-CCD} ask whether ${\cal I}$ admits a proportional, envy-free and equitable allocation, respectively. If we want to stress which disutility aggregation functions is used, we insert prefix {\sc Add} or {\sc Max} to this notation.

\iffalse
The corresponding additively approximate fair chore division problems will be denoted by  
%{\sc $c\times$-Prop-CCD} and for additively approximate versions by 
{\sc $\varepsilon$-Prop-CCD},  {\sc $\varepsilon$-EF-CCD} and  {\sc $\varepsilon$-EQ-CCD}.
\fi

It is easy to see that all the  considered  problems belong to the class \NP, as given an allocation, it can be verified in polynomial time whether it is valid and also whether it is proportional,   equitable (linear in the problem size)  or envy-free (in time $O(mn+n^2$)).

%\newpage
\section{Complete graphs}\label{sec:complete} 

Note that the classical case (without connectivity requirements) corresponds in our model to the underlying graph being complete. Then,  all the problems studied are hard  for the additive disutility aggregation. The intractability
can be proved using a reduction from \textsc{Partition} \citep[see \emph{e.g.}][]{DemkoHill98}.

\begin{prop}\label{thm:prop-2ag}
%When disutilities are encoded in binary, \todoK{Why binary?} 
{\sc Add-Prop-CCD}, {\sc Add-EF-CCD} and {\sc Add-EQ-CCD} are \NP-complete even for two agents with the same disutility function and even if the underlying graph $G$ is bipartite.
\end{prop}

Now suppose that agents aggregate their disutilities using maximum and consider the following algorithm that we shall call {\it greedy}: assign each item $v$ to agent $i$ such that $u_i(v)=\min_{j\in N}u_j(v)$. % is the minimum among all agents. 
It is easy to see that the greedy algorithm produces a valid  allocation $\pi$  that minimizes $\max_{i\in N} u_i^{max}(\pi(i))$.

If we want to decide  the existence of an equitable allocation, for each possible disutility target $w$ we do the following. 
We create the bipartite graph $H_w=(N,V,L_w)$ with the two parts of the vertex set corresponding to the set of agents and to the set of chores, respectively, and with the edge set defined by $(i,v)\in L_w$ if $u_i(v)=w$. Then we check whether $H$ admits a matching that covers $N$. If the answer is no, there is no valid allocation $\pi$ such that $u_i^{max}(\pi(i))=w$ for each agent. If the answer is yes, we try to assign the remaining chores greedily. We conclude that there exists an equitable allocation with the common disutility equal to $w$ if and only if the greedy algorithm does not assigns to any agent a chore for which she has a disutility exceeding $w$. This can be summarized as follows. 

\begin{theorem}\label{thm_max_Prop_EQ}
{\sc MAX-PROP-CCD} and  {\sc MAX-EQ-CCD} can be solved in polynomial time.
\end{theorem}

\noindent By contrast, envy-freeness criterion leads to an intractable problem. 
In the following proof we provide a reduction from  the \NP-complete problem {\sc (2,2)-e3-sat} \citep{BKS03} that asks,
given a Boolean formula $F$ in Conjunctive Normal Form, where each
clause in $F$ has size three, and each variable occurs exactly twice
unnegated and exactly twice negated, whether $F$ is satisfiable.

\begin{theorem}\label{thm_max_EF}
 {\sc MAX-EF-CCD} is \NP-complete even if the underlying graph $G$ is complete.
\end{theorem}

\begin{proof}
 Let formula $F$ %as an instance of (2,2)-E3-SAT  
 be given with a set of variables $X=\{ x_1, \ldots , x_s\}$ and  set of clauses $C=\{ c_1, \ldots , c_t\}$. Let $L$ be the set of literals i.e., $L=\bigcup_{j=1}^s\{ x_j^1, x_j^2, \bar{x}_j^1, \bar{x}_j^2\}$ and for any $\ell \in L$, let $c(\ell)$ denote the clause containing literal $\ell$. Let $L_i$ denote the set of literals in clause $c_i$.

For each formula $F$ we construct an instance of {\sc MAX-EF-CCD} as follows. The set of chores is 
$V=W\cup Z\cup Z'\cup Y$, where 
$W=\bigcup_{j=1}^sW_j$ with $W_j={\{ w_j^1, w_j^2, \bar{w}_j^1, \bar{w}_j^2\}}$ are literal chores, $Z=\{ z_1,z_2,\dots,z_s\}$ are clause chores, $Z'=\{z'_1,z'_2,\dots,z'_s\}$\ are dummy clause chores and 
$Y=\bigcup_{j=1}^sY_j$, with $Y_j ={\{ y_j^1, y_j^2, y_j^3, y_j^4\}}$ are dummy variable chores.  %The number of chores is thus $m=8s+2t$. 
The chore corresponding to  literal $\ell \in L$  will be denoted $w(\ell)$. 

We further assume that the set of chores is ordered $W_1,W_2,\dots,W_s,Z,Z',Y_1,Y_2,\dots,Y_s$ while the ordering within each subset is the same as the order in which the chores in the respective subset have been written above. The position $\beta(v)$ of a chore $v\in V$ in this ordering is equal to the disutility $u_i(v)$ derived by each agent $i\in N$ from  chore $v$, unless defined differently.

The set of agents is $N=B\cup B'\cup P\cup Q$, where $B=\{ b_1, b_2,\dots, b_t\}$ are clause agents, $B'=\{ b'_1, b'_2,\dots, b'_t\}$ are dummy clause agents, $P=\{ p_1, p_2,\dots, p_s\}$ are variable agents and $Q=\bigcup_{j=1}^sQ_j$ with $Q_j={\{ q_j^1, q_j^2, q_j^3, q_j^4\}}$ are dummy variable agents. 
%So the number of agents is $n=2t+5s+1$.

The disutilities are defined in Table \ref{tab_max_EF}, where for each agent we list the chores with disutilities equal to 0 and to $\varepsilon$, where $0<\varepsilon<1$ is fixed. The disutility of any chore $v$ to agent who does not have $v$ displayed in this table is equal to $\beta(v)$.

\begin{table}[ht]
\begin{center}
\begin{tabular}{lll}
      \toprule
      agent& disutility equal 0 & disutility equal $\varepsilon$ \\
      \midrule
      $b_i,i=1,2,\dots, t$\qquad \qquad & $z_i,\{w(\ell)\ | \ell\in L_i\}$  \qquad \qquad& -- \\
      $b'_i,i=1,2,\dots, t$\qquad \qquad & $z_i$ & $z'_i$ \\
      $p_j,j=1,2,\dots, s$\qquad \qquad & $w_j ^1,w_j ^2, \bar{w}_j^1, \bar{w}_j^2$ &-- \\
      $q_j^1,j=1,2,\dots, s$\qquad \qquad & $w_j ^1,\bar{w}_j^1$ &  $y_j^1$ \\
      $q_j^2,j=1,2,\dots, s$\qquad \qquad & $w_j ^1,\bar{w}_j^2$ &  $y_j^2$ \\
      $q_j^3,j=1,2,\dots, s$\qquad \qquad & $w_j ^2,\bar{w}_j^1$ &  $y_j^3$ \\
      $q_j^4,j=1,2,\dots, s$\qquad \qquad & $w_j ^2,\bar{w}_j^2$ &  $y_j^4$ \\
      \bottomrule
    \end{tabular}
\end{center}
\caption{Disutilities in the proof of Theorem \ref{thm_max_EF}.}\label{tab_max_EF}
\end{table}

%Julien's intuitions
Let us briefly explain how the reduction works before proving that it is correct. Each variable agent
will receive a subset of literal chores that correspond to her variable. Each dummy variable agent will
receive her corresponding dummy variable chore and will envy the corresponding variable agent as soon
as she does not receive a subset of literal chores containing either two positive literals or two negative
literals. Each clause agent will receive a clause chore as well as at least one literal chore associated with
one of her literals. Clause chores ensure that no variable agent will envy a clause agent. Furthermore, each
dummy clause agent will receive her corresponding dummy clause chore and will envy her corresponding
clause agent as soon as she does not receive at least one of her corresponding literal chores.
%end of intuition

Assume first that $f$ is a truth assignment that satisfies all clauses in $C$. We construct from $f$ an assignment $\pi$  as follows. For each variable $x_j$, if $x_j$ is false according to $f$ then $\pi(p_j)=\{w_j^1,w_j^2\}$, otherwise $\pi(p_j)=\{\bar{w}_j^1,\bar{w}_j^2\}$. Furthermore, 
$\pi(b_i)=\{z_i\}\cup \{w(\ell), \ell\in L_i$\ and\ $\ell$ is true in $f\}$.  
Finally, $\pi(b'_i)=\{z'_i\}$ for $i=1,2,\dots,t$ and $\pi(q_j^k)=\{y_j^k\}$ for $j=1,2,\dots,s$ and $k=1,2,3,4$. Clearly, $\pi$ is valid.
%and each dummy clause and dummy variable agent has a bundle with disutility $\varepsilon$ and all other agents have a bundle with disutility 0. 
It is easy to check that this allocation is envy free. Namely, the disutility  agents of $B\cup P$ receive in $\pi$ is 0, so they do not envy. Take a dummy agent $a$. This agent does not envy because $u_a(\pi(a))=\varepsilon$ and  for each agent $a'\ne a$ the bundle $\pi(a')$ contains a chore $v$ such that $u_a(v)$ is a strictly positive integer.

Conversely, suppose that there is a valid assignment $\pi$ such that no agent envies another one. %  To make our formulations easier, let us say that a chore $v$ is {\it acceptable} for agent $i$ if $u_i(v)\le \varepsilon$ and let $\pi(X)$ denote the set $\{\pi(x)\ |\ x\in X\}$ for some $X\subseteq N$. 
In the first part of the proof we use mathematical induction on the reverse ordering $\beta(v)$ of the chores in the following way:  we take the next chore $v$ and  argue that $v$ must belong to the bundle of a certain agent $i$. Let us say that agent $i$ was treated {\it treated}.

As the disutility of $y_s^4$ is maximum of all chores, if $y_s^4\in \pi(a)$ for any agent $a\ne q_s^4$ then $a$ will envy any other agent. Therefore $y_s^4\in \pi(q_s^4)$. Now suppose that $y_j^k\in \pi(q_j^k)$ for each $y_j^k\in Y$ such that $\beta(y_j^k)>u$. 
%and let us denote  the set of agents treated so far by $Y^u$.  
Take $y_j^k$ such that $\beta(y_j^k)=u$. If $y_j^k\in \pi(a)$ for some agent $a\ne q^k_j$ then $a$ will envy any other agent that was not treated yet.

%Hence we know that $Y\subseteq \pi(Q)$.
%, where we use the shorthand $\pi(X)$ to denote the set $\{\pi(x)\ |\ x\in X\}$.

Similarly, by induction  for $i=t,t-1,\dots,1$ we show that $z'_i\in \pi(b'_i).$ 
As the disutility of $z'_t$ is maximum of all chores that have not yet been assigned, if $z'_t\in \pi(a)$ for some agent $a\ne b'_t$ then $a$ will envy any other agent in $N\setminus Y$. Therefore $z'_i\in \pi(b'_i).$  Now suppose that $z'_i\in \pi(b'_i)$ for each $i>k$. Take $z'_k$. If $z'_k\in \pi(a)$ for any agent $a\ne b'_k$ then $a$ will envy any agent not treated so far. 
%$N\setminus (Y\cup\{b'_{k+1}, \dots, b'_t\})$.
 %that if $z'_i\in \color{red}\pi(a)$ for an agent $a\in A\setminus (Q\cup\{b'_{i+1},b'_{i+2},\dots,b'_t\}$\footnote{\textcolor{red}{{\bf Julien: } Isn't it even true for any agent out of this set?}} other than $b'_i$ then $a$ will envy any other agent from this set.  So $Z'\subseteq \pi(B')$. 

By an analogical inductive argument we show that   $z_i\in \pi(b'_i)$ or $z_i\in \pi(b_i)$ for each $i=t, t-1,\dots, 1$  because otherwise the agent that gets this chore will envy any untreated agent (for example, an agent in  $P$).
%unte  and that $W_j\subseteq \pi(p_j)\cup \pi(Q_j)$ for each $j=1,2,\dots,s$.

Now  we know that the disutility of each agent $q^k_j$ in $\pi$ is at least $\varepsilon$ and  $q^k_j$  does  not envy agent $p_j$. So we must have that  either  $\{w_j^1,w_j^2\}\subseteq \pi(p_j)$ or  $\{\bar{w}_j^1,\bar{w}_j^2\}\subseteq \pi(p_j)$ for each $j=1,2,\dots,s$.
Let us say that $x_j$ is {\tt false} in the former case and that $x_j$ is {\tt true} in the latter case. Finaly, so as no agent $b'_i$ envies $b_i$, we get that $\pi(b_i)$ must contain at least one chore $w(\ell)$ for $\ell\in L_i$, and due to the truth values and assignment of chores in $W$ defined above, this chore must  correspond to a true literal in clause $c_i$. Hence we obtain an assignment of truth values that  makes $F$ true.
\end{proof}

\section{Paths}\label{sec:paths}
Even if the underlying graph is restricted to be a path then all the considered chore
division problems are intractable, as we now show. All the proofs in
this section are based on the same construction starting from an instance of  {\sc (2,2)-e3-sat}. 

So let a formula $F$ as an instance of {\sc (2,2)-e3-sat} be given with variables $X=\{x_1,\dots, x_s\}$ and   clauses $C=\{c_1,\dots, c_t\}$.
By $L$ we denote the set of literals in $F$, \textit{i.e.} $L=\cup_{j=1}^s\{x_j^1, x_j^2, \bar{x}_j^1, \bar{x}_j^2\}$ and $c(\ell)$ for  $\ell\in L$ denotes the  clause that contains literal $\ell$. Notice that the structure of the formula implies $3t=4s$ and hence $s=3t/4$.

Let $c\ge 1$ be given.
We construct an instance ${\cal I}$  of CCD with the set of chores  $V=Y\cup Z\cup D$ where $Y=\{y_1,\dots,y_t\}$ are {\it clause} chores,
$Z=\cup_{j=1}^s\{z_j^1, z_j^2, \bar{z}_j^1, \bar{z}_j^2\}$ are {\it literal} chores and $D=\{d_1,\dots,d_s\}$ are  {\it variable} chores. The number of chores is thus $m=5s+t$. The graph $G$ defining the neighborhood relation  between chores (illustrated  in Figure \ref{fig1}) has the edges
\begin{itemize}\itemsep0pt
\item $(y_i,y_{i+1})$ for $i=1,2,\dots, t-1$;
\item $(y_t,z_1^1)$; 
\item $(z_j^1, z_j^2); (z_j^2, d_j), (d_j,\bar{z}_j^1),(\bar{z}_j^1,\bar{z}_j^2)$ for $j=1,2,\dots,s$;
\item  $(\bar{z}_j^2, z_{j+1}^1)$ for $j=1,2,\dots,s-1$.
\end{itemize}

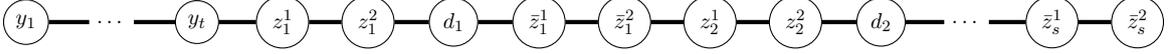
\begin{figure*}[bth]
\centering
\begin{tikzpicture}[scale=0.75, transform shape]
	\node[draw, circle](1) at (0,0) {$y_1$};
		\node(2) at (1.5,0) {$\dots$};
	\node[draw, circle](3) at (3,0) {$y_t$};
	\node[draw, circle](4) at (4.5,0) {$z_1^1$};
	\node[draw, circle](5) at (6,0) {$z_1^2$};
	\node[draw, circle](6) at (7.5,0) {$d_1$};
	\node[draw, circle](7) at (9,0) {$\bar{z}_1^1$};
	\node[draw, circle](8) at (10.5,0) {$\bar{z}_1^2$};
	\node[draw, circle](9) at (12,0) {$z_2^1$};
	\node[draw, circle](10) at (13.5,0) {$z_2^2$};
	\node[draw, circle](11) at (15,0) {$d_2$};
	\node(12) at (16.5,0) {$\dots$};
	\node[draw, circle](13) at (18,0) {$\bar{z}_s^1$};
	\node[draw, circle](14) at (19.5,0) {$\bar{z}_s^2$};

	%\draw[-, >=latex,ultra thick] (0)--(1);	
	\draw[-, >=latex,ultra thick] (2)--(1);	
	\draw[-, >=latex,ultra thick] (2)--(3);	
	\draw[-, >=latex,ultra thick] (3)--(4);	
	\draw[-, >=latex,ultra thick] (4)--(5);	
	\draw[-, >=latex,ultra thick] (6)--(5);	
	\draw[-, >=latex,ultra thick] (6)--(7);	
	\draw[-, >=latex,ultra thick] (8)--(7);	
	\draw[-, >=latex,ultra thick] (8)--(9);	
	\draw[-, >=latex,ultra thick] (10)--(9);	
	\draw[-, >=latex,ultra thick] (10)--(11);		
	\draw[-, >=latex,ultra thick] (12)--(11);		
	\draw[-, >=latex,ultra thick] (12)--(13);		
	\draw[-, >=latex,ultra thick] (14)--(13);				
\end{tikzpicture}
\caption{The graph for Lemma \ref{lemma1} \label{fig1}
}
\end{figure*}

 The set of agents in  ${\cal I}$ is  $N=B\cup P\cup Q\cup R$, where $B=\cup_{j=1}^s\{b_j^1, b_j^2, \bar{b}_j^1, \bar{b}_j^2\}$ are   {\it literal} agents, $P=\{p_1,\dots,p_s\}$, $Q=\{q_1,\dots,q_s\}$ are  {\it variable} agents and  $R=\{r_1,\dots,c_{c(5s+t)-6s+1}\}$ are dummy agents (observe that $c(5s+t)-6s+1>0$ for any $c\ge 1$).  So the  total  number of agents in ${\mathcal I}$ is $n=c(5s+t)+1$.

Let us remark that each literal $\ell\in L$ from  formula $F$ has in  ${\cal I}$ its `corresponding' chore in $Z$ and agent in $B$; they will be denoted by $z(\ell)$ and $b(\ell)$, respectively.
 
For each agent $a\in N$ her disutility is 0 for some specific chores and the total disutility of 1 for  agent $a$ is distributed uniformly among the remaining chores, to achieve normalization. In  more details:

\noindent If $a=b(\ell)\in B$ is a literal agent  then:
$$
u_{b(\ell)}(v) = \left\{
  \begin{array}{ll} 
    0  &\ \ \mbox{\ if \ } v=z(\ell) \mbox{\ or\ }  v=y_{c(\ell)} \\
    1/(5s+t-2)   &\ \ \mbox{\ otherwise}
  \end{array}
\right.
$$
\noindent If $a=p_j\in P$ then:
$$
u_{p_j}(v) = \left\{
  \begin{array}{ll} 
    0  &\ \ \mbox{\ if \ } v\in \{z_j^1, z_j^2, \bar{z}_j^1, \bar{z}_j^2\}\\
    1/(5s+t-4)   &\ \ \mbox{\ otherwise}
  \end{array}
\right.
$$
\noindent If  $a=q_j\in Q$ then:
$$
u_{q_j}(v) = \left\{
  \begin{array}{ll} 
    0  &\ \qquad \mbox{\ if \ } v=d_j\\
    1/(5s+t-1)   &\ \qquad \mbox{\ otherwise}
  \end{array}
\right.
$$
and for each agent $a\in R$ we have $u_a(v)=1/(5s+t)$ for each $v\in V$.

\begin{lemma}\label{lemma1}
If formula $F$ is satisfiable  then ${\mathcal I}$ admits a valid allocation $\pi$ such that $u_a(\pi(a))=0$ for each agent $a\in N$. If $F$ is not satisfiable then for any valid allocation $\pi$ there exists an agent whose bundle has disutility greater than $c/n$. 
\end{lemma}
\begin{proof}
Assume first that $f$  is a truth assignment of $F$  that satisfies all clauses in $C$. We construct
from $f$ a valid assignment of chores in ${\mathcal I}$ as follows. Assign each $d_j$ to $q_j$. For each variable $x_j$ assign 
agent $p_j$  the bundle $\{z_j^1,z_j^2\}$ and the agents $ \bar{b}_j^1, \bar{b}_j^2$  objects  $\bar{z}_j^1$ and $\bar{z}_j^2$, respectively if $x_j$ is true and 
assign agent $p_j$  the bundle $\{\bar{z}_j^1,\bar{z}_j^2\}$ and the agents $ b_j^1, b_j^2$  chores  $z_j^1$ and $z_j^2$, respectively if $x_j$ is false. Finally, choose the first true literal $\ell$ in each clause $c_i\in C$ and assign the corresponding literal agent $b(\ell)$ chore  $y_i$. Each agent in $R$ is assigned an empty bundle. It is easy to see that each agent receives a bundle whose disutility is 0, everybody receives a connected piece (if any) and that all chores are assigned.
 
Conversely, suppose that there is a valid assignment $\pi$ of chores in ${\mathcal I}$ such that everybody receives a bundle with disutility at most
 $c/n$. As any nonzero disutility is at least $1/(5s+t)>c/n$ this implies that each agent's bundle is either empty of has disutility 0. We now construct a truth assignment $f$ for $F$ as follows. Note first that chore $d_j$ must be assigned to agent $q_j$. Further,  for each $i$, chore $y_i$ must be assigned to some  agent $b(\ell)\in B$ that corresponds to a literal $\ell$ contained in clause $c_i$. Now suppose that for some $j$, agents $b_j^k$ as well as $\bar{b}_j^{k'}$, for $k,k'\in \{1,2\}$ are assigned some clause chores. This means that  chores $z_j^k$ as well as $\bar{z}_j^{k'}$ must both be assigned to agent $p_j$, but as chore $d_j$ is assigned to agent $q_j$ in any such assignment, agent $p_j$ gets a disconnected piece, which is a contradiction. Hence, set variable $x_j$ to be true if bundle $\{z_j^1,z_j^2\}$ is assigned to agent $p_j$ in $\pi$ and set $x_j$ to be  false if bundle $\{\bar{z}_j^1,\bar{z}_j^2\}$ is assigned to agent $p_j$. It is easy to see that this definition of truth values is consistent and makes formula $F$ true.
\end{proof}

\noindent Now we use Lemma \ref{lemma1} in the following results.

\begin{theorem}\label{thm:NPc:line:Prop:Complete}
Let $c\ge 1$ be arbitrary.  If $\Pol \ne \NP$ then there is no polynomial algorithm to decide if an instance of {\sc CCD} admits a valid %$c$-proportional 
allocation such that $u_i(\pi(i))\le c/n$ for each $i\in N$, even if $G$ is a path. Hence, {\sc Add-Prop-CCD} is \NP-complete.
\end{theorem}
\begin{proof}
As the number of agents in ${\mathcal I}$ is $n=c(5s+t)+1$, the  disutility of each agent in a $c$-proportional allocation should be  at most $1/(c(5s+t)+1)$. However, each chore with a positive disutility brings  each agent at least $1/(5s+t)$, so  the result is directly implied by  Lemma \ref{lemma1}.
\end{proof}

As $m<n$ in ${\mathcal I}$, at least one agent gets nothing. Hence in any envy-free or equitable allocation  everybody has to get a bundle with disutility 0. So Lemma \ref{lemma1} directly implies also the following assertions. 

\begin{theorem}\label{thm:NPc:line:EF:CCD}
{\sc Add-EF-CCD} and {\sc Add-EQ-CCD}  are \NP-complete if $G$ is a path. 
\end{theorem}

For the maximum-based disutility extension, let us change the construction in the beginning of this section slightly. Namely, each positive disutility of an item will be set to 1. The same arguments as above are still  valid, so we get the following assertion.

\begin{theorem}\label{thm:NPc:line:max}
 {\sc Max-EF-CCD} and {\sc Max-EQ-CCD}  are  \NP-complete if $G$ is a path, even in the binary case. 
\end{theorem}

\section{Stars}\label{sec:stars}

Let $c$ denote the center of the star. As each agent has to get a connected bundle, only the agent that is assigned $c$ can get more than one chore. According to the fairness criterion used, there are  necessary conditions that each agent has to fulfill, so as to be entitled to be assigned $c$; we shall call such an agent {\it central}. If there is no  agent who fulfills these conditions then there is no valid allocation with the desired properties. If there exists such an agent, we still have to decide about the assignment of the leaves to the other agents. For this graph topology we first present efficient algorithms and then proceed  to hard cases.

\subsection{Easy cases}

All the easy cases use a similar idea, borrowed from \cite{BouveretCechl2017}. First we guess a central agent $c$. This agent gets as many leaves as possible, and the assignment of the other leaves to other agents is found by employing an efficient matching algorithm in bipartite graphs. 

We can also perform a similar complexity analysis for all the algorithms. The upper-bound on the number of steps depends on the matching algorithm used. The bipartite graph constructed in the algorithm has $O(m+n)$ vertices and $O(mn)$ edges. The Hopcroft-Karp matching algorithm applied to a graph with $p$ vertices and $q$ edges runs in $O(q\sqrt{p})$ steps, which in our case thus means $O(mn\sqrt{m+n})$ steps. Moreover, we might need to repeat the procedure for each agent, and that leads an overall complexity of $O(mn^2\sqrt{m+n})$.

\begin{theorem}\label{thm:poly:stars:proportional}
{\sc Add-Prop-CCD} is solvable in polynomial time if $G$ is a star.
\end{theorem}
\begin{proof}
  First, a central agent $i$ must fulfill $u_i(c)\le 1/n$. Let us
  check for each such agent $i$ whether there is a proportional valid
  allocation $\pi$ assigning $c$ to $i$.

  To this end, we create a bipartite graph $H=(Z,Z',L)$ with
  $Z=N\setminus \{i\}$, $Z' = V\setminus\{c\}$ and $\{j,v\}\in L$ if
  and only if $u_j(v)\le 1/n$; the weight of this edge is $u_{i}(v)$.

  We find a maximum weight matching  $M$ in $H$.  A proportional valid allocation exists if and only if the   weight of $M$ is at least $(n-1)/n$. If this is the case, assign the
  objects to agents in $Z$ according to $M$ and all the unmatched   leaves plus the central vertex to $i$.
\end{proof}

Now let us deal with a possibility that  each agent gets a bundle with disutility 0.

\begin{prop}\label{thm:stars:binary}
If $G$ is a star then the problem to decide the existence of a valid allocation where each agent gets a bundle with disutility 0 can be solved in polynomial time.
\end{prop}
\begin{proof}
To be able to decide the existence of such an allocation,  let  us first realize that an agent $i$ can be  central  only if $u_i(c)=0$. For such an agent $i$ let $X_i=\{v\in V; u_i(v)=0\}$. Now create the bipartite graph $H=(Z,Z',L)$ where $Z=N\backslash\{i\}$, $Z'=V\backslash X_i$  and  $\{j,v\}\in L$ if $u_j(v)=0$. 
Clearly, a desired allocation exists if and only if $H$ admits a matching $M$ that covers all vertices in $Z'$.
\end{proof}

\begin{theorem}\label{thm:max:stars:EQ}
{\sc Max-EQ-CCD} is solvable in polynomial time if $G$ is a star.
\end{theorem}
\begin{proof}
Proposition \ref{thm:stars:binary} helps in deciding the 
existence of  a valid allocation such that  everybody gets disutility 0.
Now let us deal with the  case when all agents receive bundles with positive disutility equal to ${\eta}$. 

An agent $i$ can be a central agent in an equitable allocation with the common disutility equal to ${\eta}$ only if $u_i(c)\le {\eta}$. Now we proceed differently when this inequality is fulfilled as equation and when it is strict.

If $u_i(c)={\eta}$ we create a flow network  $H=(Z,L)$ as follows. Its vertices are the source $s$, sink $t$ and one vertex for each agent and one vertex for each leaf of $G$. The source is connected to each agent vertex; capacities of arcs $(s,j)$ for $j\ne i$ are 1, capacity of arc $(s,i)$ is $m-n$. There is an arc of capacity 1 between the vertex corresponding to agent $j\ne i$ and the vertex corresponding to leaf $v$ if $u_j(v)={\eta}$ and between $i$ and leaf $v$ if $u_i(v)\le \eta$.
Each leaf vertex is connected to $t$, the capacities of these arcs are also 1. A desired allocation exists if and only if there is a flow of size $m-1$ in this network, namely, the leaves are allocated to agents according to the agent-leaf arcs with nonzero flow. 

If $u_i(c)<{\eta}$ we have moreover to ensure that agent $i$ gets at least one leaf $v$   such that $u_i(v)={\eta}$. The above construction of the flow network will be modified in the following way.  Agent $i$ will not be connected with leaves directly, but there will be two more vertices $r_1,r_2$ and the following arcs:
$(i,r_1)$ with capacity $m$ and arcs $(r_1,v)$ for each leaf $v$ such that $u_i(v)={\eta}$;
arc $(i,r_2)$ with capacity $m-n-1$ and arcs $(r_2,v)$ for each leaf $v$ such that $u_i(v)<{\eta}$.
The construction of the flow network is shown in Figure~\ref{fig:max:star:EQ:flow}.
Again, a desired allocation in which agent $i$ gets at least one leaf with disutility exactly ${\eta}$ exists if and only if there is a flow of size $m-1$ in this network.

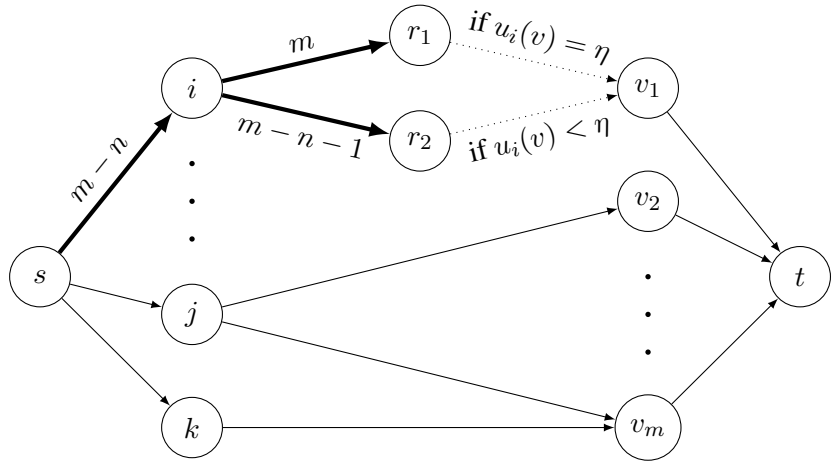
\begin{figure}[htbp]
  \centering
  \begin{tikzpicture}[>=latex]
    \draw (0, 0) node[draw, circle, minimum size=0.8cm] (s) {$s$};    

    \draw (2, 2.5) node[draw, circle, minimum size=0.8cm] (i) {$i$};    
    \draw (2, 1.5) node[fill, circle, minimum size=2pt, inner sep=0pt] {};
    \draw (2, 1) node[fill, circle, minimum size=2pt, inner sep=0pt] {};
    \draw (2, 0.5) node[fill, circle, minimum size=2pt, inner sep=0pt] {};
    \draw (2, -0.5) node[draw, circle, minimum size=0.8cm] (j) {$j$};    
    \draw (2, -2) node[draw, circle, minimum size=0.8cm] (k) {$k$};    

    \draw (5, 3.2) node[draw, circle, minimum size=0.8cm] (r1) {$r_1$};    
    \draw (5, 1.8) node[draw, circle, minimum size=0.8cm] (r2) {$r_2$};    

    \draw (8, 2.5) node[draw, circle, minimum size=0.8cm] (v1) {$v_1$};    
    \draw (8, 1) node[draw, circle, minimum size=0.8cm] (v2) {$v_2$};    
    \draw (8, -1) node[fill, circle, minimum size=2pt, inner sep=0pt] {};
    \draw (8, -0.5) node[fill, circle, minimum size=2pt, inner sep=0pt] {};
    \draw (8, 0) node[fill, circle, minimum size=2pt, inner sep=0pt] {};
    \draw (8, -2) node[draw, circle, minimum size=0.8cm] (vm) {$v_m$};    

    \draw (10, 0) node[draw, circle, minimum size=0.8cm] (t) {$t$};

    \draw[->, ultra thick] (s) -- (i) node[pos=0.5, above, sloped] {$m-n$};
    \draw[->] (s) -- (j);
    \draw[->] (s) -- (k);
    \draw[->, ultra thick] (i) -- (r1) node[pos=0.5, above, sloped] {$m$};
    \draw[->, ultra thick] (i) -- (r2) node[pos=0.5, below, sloped] {$m-n-1$};
    \draw[->, dotted] (r1) -- (v1) node[pos=0.5, above, sloped] {if $u_i(v) = \eta$};
    \draw[->, dotted] (r2) -- (v1) node[pos=0.5, below, sloped] {if $u_i(v) < \eta$};
    \draw[->] (j) -- (v2);
    \draw[->] (j) -- (vm);
    \draw[->] (k) -- (vm);
    \draw[->] (v1) -- (t);
    \draw[->] (v2) -- (t);
    \draw[->] (vm) -- (t);
  \end{tikzpicture}
  \caption{All the thin arcs have capacity equal 1. Capacity of thick arcs is shown next to them. The condition for dotted
arcs to be included is shown next to them.}
\label{fig:max:star:EQ:flow}
\end{figure}
\end{proof}

As we shall see later, the problem of deciding the existence of an envy-free valid allocation in the additive case is in general \NP-complete. However, there is a plausible efficiently solvable special case. We say that {\it agents' preferences are strict on chores} if $u_i(v)\ne u_i(w)$ for any agent $i$ and any pair of distinct vertices $v,w$. 

\begin{theorem}\label{thm:Add:stars:EF:strict}
{\sc Add-EF-CCD} is  solvable in polynomial time if $G$ is a star and the agents' preferences are strict on chores. 
\end{theorem}
\begin{proof}
If  $m<n$ then each agent has to get a bundle with disutility 0, and since preferences are strict, this is only possible if each agent gets at most one vertex, namely the one where she has disutility 0. We can easily  verify whether this happens by matching techniques.

Now let us proceed to the  case when agents receive bundles positive disutility.
Let us in turn check for each agent $i$ whether she can be the central agent. To this end, 
order all the vertices in $V$ according to $i$'s disutility increasingly (we might possibly rename the  the vertices in $G$)
\begin{equation}\label{e_star1}
%u_i(v_1)\le u_i(v_2)\le \dots u_i(v_{m-n+1})\le u_i(v_{m-n+2})\le \dots \le u_i(v_m).
u_i(v_1)<  \dots < u_i(v_{m-n+1})< u_i(v_{m-n+2})< \dots < u_i(v_m)
\end{equation}
and notice that the central agent will receive exactly $m-n+1$ vertices. The necessary conditions a central agent $i$ must therefore fulfill  are:
\begin{itemize}\itemsep0pt
\item[$(i)$] vertex $c$ must be among the vertices $v_1,v_2,\dots, v_{m-n+1}$ and
\item[$(ii)$]  $\sum_{k=1}^{m-n+1}u_i(v_k)\le u_i(v_{m-n+2})$
\end{itemize} 
These conditions imply that when assigned the bundle $X=\{v_1,v_2,\dots, v_{m-n+1}\}$, agent $i$ will not envy any other agent. We still have to ensure no envy among other agents.

Let us construct the bipartite graph $H=(Z,Z',L)$ with $Z=N\backslash\{i\}$, $Z'=V\backslash X$  and  $\{j,v\}\in L$ if the following inequality is fulfilled:
\begin{equation}\label{cond_add}
u_j(v)\le min\{\{u_j(v'); v'\in Z'\}, u_j^{add}(X)\}.
\end{equation}
Finally, there exists in $\I$ an envy-free assignment  if and only if $H$ admits a perfect matching $M$. 
Namely, agent $i$ is assigned bundle $X$ and the other agents chores according to $M$.
\end{proof}

The proof of the previous theorem can easily be adapted also for the maximum based disutility aggregation.

\begin{theorem}\label{thm:Max:stars:EF}
{\sc Max-EF-CCD} is  solvable in polynomial time if $G$ is a star and the preferences of agents are strict on chores. 
\end{theorem}
\begin{proof}
When checking whether agent $i$ cen be central, after we permute the vertices of $G$ according to $i$'s disutility increasingly, it is easy to see that it suffices to check just condition $(i)$. Further, in  the construction of the bipartite graph $H$ condition (\ref{cond_add}) should be replaced by
\begin{equation}\label{cond_max}
u_j(v)\le min\{\{u_j(v'); v'\in Z'\}, u_j^{max}(X)\}
\end{equation}
and the rest of the proof follows.
\end{proof}

\subsection{Hard cases} 

It turns out that the only intractable problems when the underlying graph is the star are connected with envy-freeness and equitability  when the disutilities are aggregated additively.

\begin{theorem}\label{thm:Add:stars:EF}
%  When disutilities are encoded in binary,\todoK{Why we need binary?}
 \textsc{Add-EF-CCD} is
  \NP-complete even if the underlying graph $G$ is a star.
\end{theorem}

\begin{proof}
  The reduction is from \textsc{(2,2)-e3-sat}. A boolean formula $F$ has the set of variables $X=\{ x_1, \ldots , x_s\}$ and the set
  of clauses $C=\{ c_1, \ldots , c_t\}$. By $L$ we denote the set of
  literals, i.e.
  $L=\bigcup_{j=1}^s\{ x_j^1, x_j^2, \bar{x}_j^1, \bar{x}_j^2\}$ and
  $c(\ell)$ for $\ell \in L$ denote the clause that contains literal
  $\ell$. Let $L_i$ denote the subset of literals appearing in clause
  $c_i$.

  We construct an instance of \textsc{Add-EF-CCD} with $m=7s+2$ chores
  and $n=2s+t+2$ agents defined as follows. The set of chores is
  $V = Y\cup Z\cup \{ c, d\}$, where $c$ denotes the center of the
  star $G$, $Y=\bigcup_{i=1}^s\{y_i, \bar{y}_i, \tilde{y}_i\}$ are
  variable chores and
  $Z=\bigcup_{j=1}^s\{ z_j^1, z_j^2, \bar{z}_j^1, \bar{z}_j^2\}$ are
  literal chores. Observe that each literal $\ell \in L$ has its
  corresponding chore which will be denoted by $z(\ell)$.

  The set of agents is $N=B \cup P\cup Q\cup \{ e, r\}$, where
  $B=\{ b_1, \ldots, b_t\}$ are clause agents and
  $P=\{ p_1, \ldots, p_s\}$, $Q=\{ q_1, \ldots, q_s\}$ are variable
  agents.
  
  The utility functions are defined as follows.

  If $a=b_i\in B$ is a clause agent then:
  \[
    u_{b_i}(v) = \left\{
      \begin{array}{ll} 
        0  &\ \ \text{\ if \ } v=d \text{\ or\ }  v=z(\ell) \text{ for some } \ell \in L_i \\
        1/(7s-2)   &\ \ \text{\ otherwise.}
      \end{array}
    \right.
  \]

  If $a=p_i\in P$ then:
  \[
    u_{p_i}(v) = \left\{
      \begin{array}{ll} 
        0  &\ \ \text{\ if \ } v=y_i \text{\ or\ }  v=\bar{z}_i^j \text{ for some } j \in \{ 1, 2\} \\
        \varepsilon &\ \ \text{\ if \ } v=\tilde{y}_i \text{\ or\ } v=d\\
        (1-\varepsilon)/(7s-3)   &\ \ \text{\ otherwise,}
      \end{array}
    \right.
  \]
  
  where $\varepsilon$ is such that $\varepsilon < (1-\varepsilon)/(7s-3)$.

  If $a=q_i\in Q$ then:
  \[
    u_{q_i}(v) = \left\{
      \begin{array}{ll} 
        0  &\ \ \text{\ if \ } v=\bar{y}_i \text{\ or\ }  v=z_i^j \text{ for some } j \in \{ 1, 2\} \\
        \varepsilon &\ \ \text{\ if \ } v=\tilde{y}_i \text{\ or\ } v=d\\
        (1-\varepsilon)/(7s-3)   &\ \ \text{\ otherwise.}
      \end{array}
    \right.
  \]
  
  If $a=e$ then:
  \[
    u_{e}(v) = \left\{
      \begin{array}{ll} 
        1/(s+1)  &\ \ \text{\ if \ } v=d \text{\ or\ }  v=\tilde{y}_i \text{ for some } j \in \{ 1, \ldots , s\} \\
        0  &\ \ \text{\ otherwise.}
      \end{array}
    \right.
  \]
  
  If $a=r$ then:
  \[
    u_{r}(v) = \left\{
      \begin{array}{ll} 
        0  &\ \ \text{\ if \ } v=d\\
        1/(7s+1)  &\ \ \text{\ otherwise.}
      \end{array}
    \right.
  \]

  Assume first that $f$ is a truth assignment that satisfies all
  clauses in $C$. We construct from $f$ a valid assignment of chores
  as follows. For each variable $x_i$, if $x_i$ is true, assign
  $\tilde{y}_i$ and $\bar{y}_i$ to $p_i$ and $q_i$ respectively, and
  if $x_i$ is false, assign $y_i$ and $\tilde{y}_i$ to $p_i$ and
  $q_i$. Furthermore, for each clause $c_i$ pick literal
  $\ell \in L_i$ which is true and assign $z(\ell)$ to $b_i$. Finally,
  assign $d$ to $r$, and the remaining chores to $e$. It is easy to
  check that this allocation is envy-free.

  Conversely, suppose that there is a valid assignment $\pi$ of chores
  such that no agent envies another one. Only the central agent can
  receive strictly more than one chore, and, as a consequence, this
  central agent must receive at least $5s-t+1$ chores (otherwise some
  chores will be left unassigned). Suppose that the central agent is
  not $e$. Then this agent will receive at least two chores with the
  highest disutility, and hence will envy any other agent receiving just
  one chore. Therefore, only $e$ can be the central agent.

  Suppose that $e$ receives a non-zero disutility. Then she will envy
  any other agent receiving either one chore other than $d$ or one
  $\tilde{y}_i$ or no chore at all. It thus means that $d$ and each
  chore $\tilde{y}_i$ must be assigned to some other agents.

  Since $\tilde{y}_i$ provides a non-zero disutility to every agent,
 no agent can receive an empty bundle; otherwise each agent that
  receives chore $\tilde{y}_i$ will be envious.

  If $d$ is not assigned to $r$ then $r$ will envy the agent who
  receives it. Hence, $d$ is assigned to $r$ in $\pi$. This in turn
  implies that each agent $b_i$ should receive chore $z(\ell)$ for
  some $\ell \in L_i$. This also implies that each variable agent
  should receive a chore for which she has disutility 0 or
  $\varepsilon$, which means that chore $\tilde{y}_i$ is either
  assigned to agent $p_i$ or $q_i$.

  We now construct truth assignment $f$ as follows. If $\tilde{y}_i$
  is assigned to $p_i$ in $\pi$ then set $x_i$ to be true, and
  otherwise (i.e. $\tilde{y}_i$ is assigned to $q_i$) set $x_i$ to be
  false. We will show that $f$ satisfies each clause $c_i$. Let $\ell$
  be the literal of $c_i$ such that $z(\ell)$ is assigned to
  $b_i$. Assume that $\ell$ is a positive literal of variable $x_j$
  (the negative case can be treated in a similar way). If
  $\tilde{y}_j$ is assigned to agent $q_i$ then she will envy agent
  $b_i$, leading to a contradiction. Therefore, $\tilde{y}_i$ is
  assigned to agent $p_i$, $x_i$ is set to true and clause $c_i$ is
  true.
\end{proof}

The equitability criterion also leads to an \NP-complete problem.

\begin{theorem}\label{thm:NPc:star:EQ}
{\sc Add-EQ-CCD} is \NP-complete if $G$ is a star. 
\end{theorem}
\begin{proof}
We shall provide a polynomial reduction from the following version of the \NP-complete problem {\sc partition} \citep{Garey79}; symbol $[p]$ denotes the set $\{1,2,\dots,p\}$.
\begin{quotation}
\noindent {\bf Instance ${\cal J} $:} The set $\{a_i,b_i; i \in [p]\}$ of integers such that $\sum_{i\in[p]}(a_i+b_i)=2K$.\\
{\bf Question:} Does there exist a partition $(P,P')$ of $[p]$ such that  $\sum_{i\in P}a_i+\sum_{i\in P'}b_i=K$?\\
\end{quotation}
Let us construct an instance ${\cal I}$ of {\sc Add-EQ-CCD} as follows. The set of chores is $V=\{c,d\}\cup V'$, where $V'=$ $\{v_i,w_i; i \in [p]\}$. Chore $c$ is the center of the star, the other chores are its leaves. The set of agents is $N=\{j_0,j_{p+1}\}\cup N'$, where $N'=\{j_i; i\in [p]\}$. The disutilities of agents are as follows and it can be easily checked that they are all normalized.
$$u_{j_0}(v)=\left\{\begin{array}{ll}
1/6 & \text{\ for\ } v=c\\
1/2  & \text{\ for\ } v=d\\
a_i/(6K) & \text{\ for\ } v=v_i; i\in[p]\\
b_i/(6K) & \text{\ for\ } v=w_i; i\in[p]
\end{array}\right.
$$
$$u_{j_{p+1}}(v)=\left\{\begin{array}{ll}
1/3 & \text{\ for\ } v=d\qquad\quad \\
2/(6p+3)  & \text{\ otherwise}
\end{array}\right.
$$
and, finally,  for $i\in[p]$
$$u_{j_i}(v)=\left\{\begin{array}{ll}
1/3 & \text{\ for\ } v\in\{c,v_i,w_i\}\quad\\
0 & \text{\ otherwise}
\end{array}\right.
$$

Now suppose that $(P,P')$ is a partition witnessing that  ${\cal J}$ is a yes instance; let us define a valid assignment in the following way.
$\pi(j_0)=\{c\}\cup \{v_i, i\in P\}\cup \{w_i, i\in P'\}$ and $\pi(j_{p+1})=\{d\}$.
Further,
$\pi(j_i)=\{v_i\}$ if $i\in P'$ and $\pi(j_i)=\{w_i\}$ if $i\in P$. It is easy to see that $\pi$ is a  valid assignment and for each agent the disutility of her piece is $1/3$.

Conversely, suppose that ${\cal I}$ admits a valid equitable  assignment such that each agent  gets the same disutility equal to ${\eta}$. First, let us realize that ${\eta}\ne 0$, as the agent who receives the central vertex $c$ has a positive disutility.  Further, ${\eta}$ cannot be greater than $1/3$. Namely,  as exactly one agent  can be assigned the central vertex $c$ (and hence more than one chore), we would not be able to give  to each agent in $N'$ a piece with disutility greater than $1/3$. This also immediately implies that the only agent who can receive  chore $c$ is agent $j_0$ and that ${\eta}=1/3$. So agent $j_{p+1}$ receives chore $d$ and each agent in $N'$ receives either chore $v_i$ or chore $w_i$. This means that $\pi(j_0)=\{c\}\cup \{v_i; i\in P\} \cup \{w_i; i\in P'\}$ where $(P,P')$ is a partition of $[p]$; moreover, as disutility $j_0$ derived from $\pi(j_0)$ is equal to $1/3$, we must have
$$\sum_{i\in P}a_i/(6K)+\sum_{i\in P'}b_i/(6K)=1/6$$
which implies
$$\sum_{i\in P}a_i+\sum_{i\in P'}b_i=K$$
and hence $(P,P')$ is a partition of $[p]$ verifying that ${\cal J}$ is a yes instance of {\sc partition}. 
\end{proof}

\section{Conclusion and open problems}\label{sec:conclusion}

In this paper we studied the computational complexity of the problem of finding a fair allocation of nondisposable undesirable items (chores) with the additional requirement that each agent has to receive a bundle of chores that is  connected in the graph representing  the relationship between items. 

We have demonstrated that the chore division problems and their corresponding  ``dual'' fair division  problems do not necessarily have solutions that directly translate from one context to another, moreover, the computational complexity of the corresponding problems can differ. 

We proposed polynomial algorithms for some existence problems and showed that other problems are \NP-complete. Moreover, our construction for paths even leads to multiplicative inapproximability results.

Notice that \cite{Suksompong17} considered   fair allocations on paths with contiguous bundles that are approximately fair up to an additive factor. It is easy to see that his constructions could be, with very minor modifications,  applied also in the chore-division problem. However, we do not know how to construct additively approximately fair valid allocations for the other simple graph, the star. 

Other natural  open questions can be thought of.
Is there any graph structure that could separate the polynomial problems from the intractable ones for various fairness criteria? 
%Is there a type of graph that allows a polynomial algorithm for equitability criterion?
When no valid fair allocation exists, one can think of some relaxations of the
connectivity constraints. One could  for example ask that  the share of each agent should consist of not more than $k$  disconnected pieces or that each diameter should be bounded.

Further,  we have omitted the recently introduced fairness criteria maximin share guarantee and envy-freeness up to one good. We believe that they may lead to some more interesting results.

\section*{Acknowledgements}
This work has been supported by the bilateral Slovak-French project APVV SK-FR-2017-0022.

\bibliography{cakebib}

\end{document}